\documentclass{article}

\usepackage{iclr2027_conference,times}

\usepackage[english]{babel}
\usepackage[utf8]{inputenc}

\usepackage{amsmath,amssymb}
\usepackage{amsthm}
\usepackage{mathtools}
\usepackage{bbm}
\usepackage{mathrsfs}

\usepackage{graphicx}
\graphicspath{{figures/}}

\usepackage{tabularx}
\usepackage[table]{xcolor}
\usepackage{tcolorbox}
\usepackage{listings}
\usepackage{float}
\usepackage{enumitem}
\usepackage{nicefrac}
\usepackage{gensymb}
\usepackage{algorithm}
\usepackage{algpseudocode}
\usepackage{comment}
\usepackage{booktabs}
\usepackage{array}
\usepackage{longtable}
\usepackage{ragged2e}
\usepackage{subcaption}

\usepackage{hyperref}
\hypersetup{hidelinks}
\usepackage{url}

\newcommand{\Def}{\overset{\text{def}}{=}}

\theoremstyle{plain}
\newtheorem{proposition}{Proposition}

\theoremstyle{plain}

\theoremstyle{plain}

\title{
Learned Enterprise Data Comprehension: Compression and Routing for Data Agents
}

\author{Ethan Torres\\
permute.ai\\
\texttt{ethanjt2@illinois.edu}
\And
Eric Mills\\
permute.ai\\
\texttt{eric@permute.ai}\\
\texttt{me@ericmills.io}
}

\iclrfinalcopy

\begin{document}

\maketitle

\begin{abstract}
Structured-data agents in enterprise settings must reason over complex data environments whose relevant evidence is distributed across schemas, relationships, policies, and recurring business roles. Modern agentic systems often address this burden through reusable markdown-style memory or skill files that preserve previously discovered information for later queries, reducing the need to rediscover the same structure repeatedly. This is useful, but it obscures a natural division of labor: agents are well suited to semantic reasoning, while learned systems are well suited to predicting and organizing recurring structure. We introduce \emph{latent equivalence learning} to bridge this gap. The framework separates persistent task-relevant identities from their dataset-relative realizations. In our realization, supporting and opposing evidence shape support-realized Gaussian prototypes that learn how those identities are expressed in a particular data environment, while soft-membership profiles retain distinctions lost under a hard assignment. A separate learned query-prototype system represents recurring evidential requirements and maps them through a learned compatibility function into the same persistent identity structure. This identity-factorized, query-conditioned routing materializes the relevant dataset-specific evidence for downstream reasoning, allowing the agent to operate over an already organized evidential state rather than reconstructing cross-schema structure at every query. On the Data Agent Benchmark, spanning 54 queries across 12 heterogeneous datasets, our full implementation achieves \(94.67\%\) dataset-macro stratified Pass@1 over five complete trials and \(258/270\) successful raw query attempts, compared with \(55.51\%\) for the benchmark's Claude Opus~4.6 reference agent, ranking first among 40 leaderboard entries at submission.

\end{abstract}


\section{Introduction}
\label{sec:introduction}

As foundation models and agent frameworks become broadly accessible, organizations increasingly face a different bottleneck: not obtaining a capable model, but exposing the relevant internal evidence to it. This problem is especially relevant for data agents, which answer natural-language questions over heterogeneous organizational data. Consider a company whose two products are managed by different
departments, recording customers under \texttt{buyer\_id} in one system and \texttt{customer\_id} in the other. To answer ``Which customers use both products?'', an agent must identify the shared customer-identifier role, verify which records refer to the same customer, and combine the
relevant data. The goal is to answer the company-wide question without requiring the user to reconcile the departmental schemas first. Such requests combine evidence discovery, retrieval, and execution \citep{ma2026dab}. We refer to the resulting systems challenge as the \emph{three Cs}: \emph{computation}, the work required to locate and
use relevant evidence; \emph{cost}, the training and inference resources consumed; and \emph{context}, the evidence that can be exposed to a language model at once.

Agentic decomposition alone does not resolve these problems. Adding more agents, calls, or reasoning steps can increase the amount of exploration without improving the organization of the evidence being explored. Recent benchmarks show that heterogeneous data questions remain difficult for frontier agents and that increasing the interaction horizon alone does not reliably produce successful long-horizon data analysis \citep{ma2026dab,xu2026longdsbenchfailurelonghorizonagentic}. At the opposite extreme, assigning the entire burden to a monolithic learned predictor is not automatically better. High-dimensional feature spaces make variable selection and representation design increasingly consequential \citep{guyon2003introduction,fan2006statistical}; large training and tuning pipelines impose substantial computational and energy costs \citep{strubell2019energy,schwartz2020greenai}; and models deployed in nonstationary settings remain vulnerable to distribution shift, degraded uncertainty calibration, and interference during continual adaptation \citep{koh2021wilds,ovadia2019uncertainty,parisi2019continual}.

We argue that these difficulties partly reflect a mismatch in the division of labor between learned systems and language-model agents. In enterprise data environments, language-model agents are commonly asked both to discover the latent organization of an unfamiliar data system and to reason over that organization. Conversely, task-specific learned models are often expected to encode the full semantics of every possible query, schema, and execution path. Neither allocation uses the respective system where it is strongest.

\textbf{Our central thesis is that enterprise data agents should reason over learned dataset-relative realizations of persistent task-relevant identities, rather than rediscovering those identities from dataset-specific surface forms at every query.} We introduce \emph{latent equivalence learning} as the representation-learning framework for constructing dataset-relative realizations of persistent identities from supporting and opposing evidence, together with a \emph{query-conditioned routing system} that arises naturally once the learned evidence is organized around those shared identities. The routing system consumes this learned identity structure rather than reconstructing it: recurring query requirements are mapped to the persistent identities they require, and those identities in turn expose their realizations and associated evidence in the current dataset. The framework itself does not prescribe a particular prototype family or evidence-construction mechanism; the Gaussian-prototype system studied here is one concrete realization.

Our earlier example of a business query spanning multiple departmental representations of a customer identifier exposes the distinction that motivates the framework. The recurring role \emph{customer identifier} is a persistent identity: its task-relevant meaning can be reused even when one system records it as \texttt{buyer\_id}, another as \texttt{customer\_id}, and their surrounding schemas differ. The departmental columns, records, relations, and other observed structures are instead evidence objects from which the dataset-specific realization of that identity must be learned. In this sense, the business role is held fixed across the modeled environments, while its realization is latent: learning determines how observed objects express that role from the available supporting and opposing evidence.

This separation extends beyond a single company. Organizations in different industries may instantiate comparable business roles using different names, schemas, and local features. For a scoped deployment domain, we model these recurring roles by a finite inventory of task-relevant identities and learn how each dataset realizes them, rather than attempting open-world semantic equivalence. The representation is trained to absorb variation that is consistent with a shared identity while preserving distinctions that the available evidence indicates remain task-relevant. Where the same roles recur, the identity structure can be reused while dataset-relative realizations are learned locally; genuinely new roles require extending the modeled inventory.

We instantiate this framework with support-realized Gaussian prototypes
\citep{snell2017prototypical,fort2017gaussian}. The evidence-side model learns
the dataset-relative realizations of the persistent identities just
described, while a separate learned query-prototype model represents
recurring evidential requirements. A learned compatibility map converts
those requirements into activations over the persistent identity
coordinates. In the example, a request concerning customers who use both
products activates customer-identifier and product-relation requirements;
the current dataset realization then determines which departmental columns,
records, and relations provide evidence for those identities. A bounded
view of the supporting, opposing, and competing evidence passes to the
semantic Arbiter and tool-equipped data agent for interpretation,
verification of customer-level correspondence, and computation. Query
requirements can therefore be reused across environments even though the
evidence that realizes them changes.

We then retain the full soft profiles because evidence or queries with the same
hard identity may differ in secondary memberships that change which competing evidence should be surfaced. Separately, the Gaussian evidence-side realization admits a local assignment-stability bound characterizing when perturbations of an embedding preserve its hard identity under a fixed prototype bank. Finally, we instantiate the framework in a data-agent harness and evaluate it on the Data Agent Benchmark (DAB), which tests agents across heterogeneous schemas and multi-step data questions~\citep{ma2026dab}. The evaluation tests whether
adding the learned equivalence representation improves a strong agentic system under a fixed downstream reasoner.

\paragraph{Contributions.}
We formulate latent equivalence learning through persistent identities, dataset-relative realizations, and abstract supporting and opposing evidence. We instantiate it with Gaussian evidence prototypes and a separate query-prototype system connected by a learned compatibility map. We
characterize the resulting identity-factorized routing and derive a conditional local assignment-stability certificate, including its connection to the existing signed-evidence losses. We evaluate the resulting representation-and-routing system in a constrained data-agent harness.

\begin{figure}[H]
    \centering
    \includegraphics[width=1\linewidth]{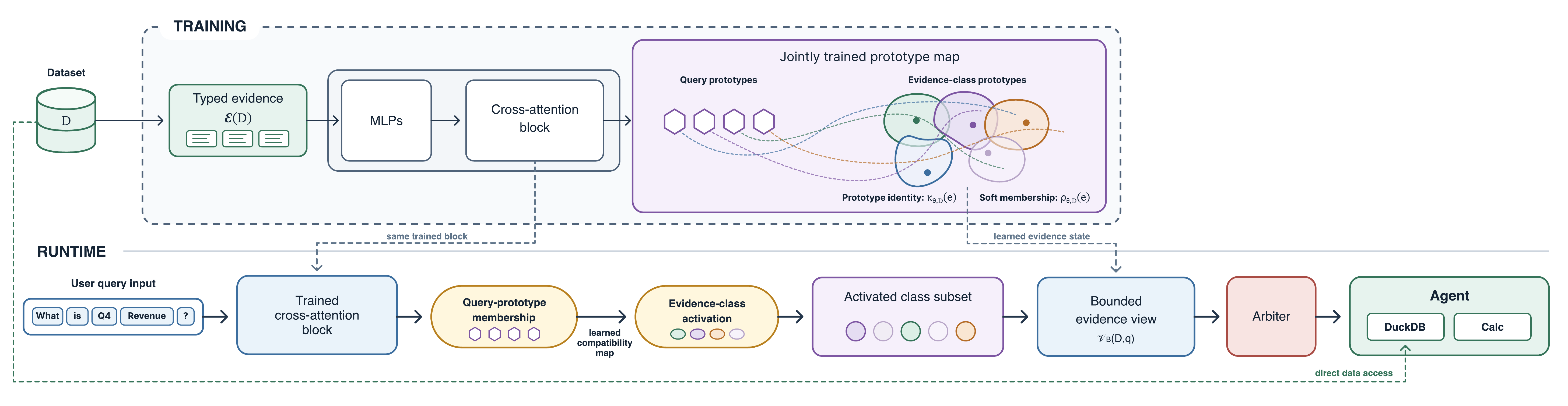}
\caption{\textbf{Training and runtime of the proposed architecture.}
Offline, heterogeneous evidence is encoded into dataset-relative prototype realizations of persistent identities, while a separate query-prototype system learns recurring evidential requirements and their compatibility with those identities. At runtime, a request activates persistent identity coordinates, the current dataset materializes a bounded view of their associated evidence, and the Arbiter and constrained data agent perform semantic interpretation and execution over that view.}
    \label{fig:equivalence_overview}
\end{figure}

\section{Related Work}

\paragraph{Invariance and prototype learning.}
Deep Sets builds permutation invariance into set representations, while
large-margin metric learning learns distances that favor same-label
neighbors over competing labels
\citep{zaheer2017deepsets,JMLR:v10:weinberger09a}.
Prototypical Networks represent episodic classes by support-derived centers;
Gaussian Prototypical Networks extend this construction with precision-weighted
support and class-dependent distances
\citep{snell2017prototypical,fort2017gaussian}.
We adopt these ingredients to realize persistent evidence identities for a
particular dataset and connect them to a separate query-prototype system.
Our local assignment certificate specializes the established relation
between classification margins and Lipschitz sensitivity
\citep{tsuzuku2018lipschitz}; the general certification principle is not
claimed as a new contribution.

\paragraph{Structured correspondence and weak supervision.}
Sherlock and SATO infer semantic column types using column and table
context, while Snorkel combines imperfect labeling sources into training
supervision
\citep{hulsebos2019sherlock,10.14778/3407790.3407793,Ratner_2017}.
Classical structured-data methods provide complementary inductive biases:
Lasso favors sparse linear support, while XGBoost captures nonlinear
feature interactions through boosted trees
\citep{tibshirani1996lasso,chen2016xgboost}.
These approaches illustrate how heterogeneous representations, predictors,
and imperfect supervision can expose different evidence about the same
structured environment. In our formulation, such outputs may contribute
supporting or opposing evidence, but they do not themselves define the
persistent identities. The distinction of interest is instead between a
persistent identity and its dataset-relative realization, whose memberships
and associated evidence remain available for subsequent query-conditioned use.

\paragraph{Evidence routing and data agents.}
Mixture-of-experts models route inputs to specialized computation, and
retrieval-augmented generation conditions language-model generation on
retrieved evidence
\citep{jacobs1991adaptive,shazeer2017outrageously,
lewis2021retrievalaugmentedgenerationknowledgeintensivenlp}.
StructRAG constructs task-appropriate representations of retrieved
information before reasoning \citep{li2025structrag}.
At the agent level, Data Interpreter uses hierarchical graph modeling and
programmable node generation to organize and execute data-science workflows
\citep{hong2024datainterpreter}.
Our focus is the learned interface preceding agent execution: query-prototype
memberships activate persistent evidence identities, and a dataset-specific
map materializes the corresponding supporting and opposing evidence.
This factorization distinguishes reusable query requirements from their
realization in a particular data environment; the downstream agent retains
responsibility for semantic interpretation and execution.

\section{Latent Equivalence Learning}
\label{sec:method}

We consider a scoped enterprise domain in which recurring task-relevant roles can be represented by a finite identity set shared across the modeled datasets. This separates what the framework holds fixed from what must adapt to each environment: the identity meanings provide the common reference structure, while learning determines how those identities are realized in a dataset, how observed evidence is assigned across them, and how a query activates them. We first formalize the evidence interface, then instantiate the dataset-relative realization with Gaussian prototypes, define the query-conditioned routing, and
analyze assignment stability. Parameter dependence is suppressed once the
trained model is fixed.

\subsection{Problem Setup and Signed Evidence}
\label{sec:setup}

The framework must distinguish the task-relevant roles that are shared across the modeled domain from the observed structures through which a particular dataset expresses them. Let $D$ denote a structured data environment available to the system, and let $\mathcal{E}(D)$ be a finite set of \emph{evidence objects} extracted from it. An evidence object $e\in\mathcal{E}(D)$ is an observed structured object---for example, a column, record group, or relation together with its available context---whose task-relevant role is to be represented by the model.
Let $\mathcal{C}$ be a finite, nonempty set indexing the task-relevant roles covered by the modeled domain. We call these \emph{persistent identities} because the meaning associated with each $c\in\mathcal{C}$ is shared across dataset realizations, even though the evidence that realizes it may change with $D$. Thus $c$ denotes the shared role, whereas its dataset-relative realization will be learned from the evidence in $\mathcal{E}(D)$. To represent how an observed evidence object bears on a candidate identity, let
\[
    \omega_D^{+},\omega_D^{-}:
    \mathcal{E}(D)\times\mathcal{C}
    \longrightarrow
    \mathbb{R}_{\geq 0}
\]
For $(e,c)\in\mathcal E(D)\times\mathcal C$,
$\omega_D^{+}(e,c)$ measures evidence supporting the assignment of
evidence object $e$ to identity $c$, while $\omega_D^{-}(e,c)$ measures
evidence opposing that assignment. These maps form the interface between
application-specific evidence construction and latent equivalence learning: the framework specifies how their outputs are used, but not how a particular application must produce them. In the Gaussian realization below, supporting weights contribute to prototype construction and attraction, whereas opposing weights penalize proximity to incompatible identity realizations. Thus supporting evidence determines which observed variation may be aligned, while opposing evidence preserves distinctions that remain task-relevant.

\subsection{Soft Membership and Assignment Fibers}
\label{sec:fibers}

The signed evidence interface describes how individual observations support or
oppose candidate identities; we next need a representation that records the
resulting assignment structure without immediately collapsing it to a single
label. Let $\Delta(\mathcal C)\Def\left\{\pi\in\mathbb R_{\geq 0}^{|\mathcal C|}:\sum_{c\in\mathcal C}\pi_c=1\right\}$ be the probability simplex indexed by the persistent identities, and let $\rho_D:\mathcal E(D)\longrightarrow\Delta(\mathcal C)$ assign each evidence object a normalized soft-membership profile. The coordinate $[\rho_D(e)]_c$ records the relative membership of evidence object $e$ in identity $c$; Section~\ref{sec:prototypes} gives the learned Gaussian realization used in our experiments. To obtain a single identity when a hard readout is required, fix a deterministic
selector
\[
    T:
    2^{\mathcal C}\setminus\{\varnothing\}
    \longrightarrow
    \mathcal C,
    \quad
    T(A)\in A
\]
for every nonempty $A\subseteq\mathcal C$. Define
\begin{equation}
\begin{aligned}
\kappa_D(e)
&\Def
T\!\left(
    \operatorname*{arg\,max}_{c\in\mathcal C}
    [\rho_D(e)]_c
\right),\\
e\sim_D e'
&\quad\Longleftrightarrow\quad
\kappa_D(e)=\kappa_D(e'),
\qquad
Q_{c,D}\Def\kappa_D^{-1}(\{c\}).
\end{aligned}
\label{eq:assignment-fibers}
\end{equation}
Thus $\sim_D$ is an equivalence relation on $\mathcal E(D)$ induced by the
hard identity readout, and its nonempty fibers partition the evidence objects.
The identity $c$ is the shared role label, whereas $Q_{c,D}$ is the
dataset-dependent set of evidence objects assigned to that role. The hard fibers define equivalence classes, but they discard how strongly an
object competes among identities. We therefore retain the full profile
$\rho_D(e)$ so that later evidence materialization can distinguish objects
with the same hard assignment but different secondary memberships. A natural
comparison on these profiles is
\[
 d_D^{\mathrm{mem}}(e,e')
 \Def
 \|\rho_D(e)-\rho_D(e')\|_2 .
\]
This is a pseudometric: zero distance means identical membership profiles and
therefore the same hard identity, while the converse need not hold. The
Gaussian realization below learns $\rho_D$ from prototype distances; the
elementary quotient and pseudometric properties are verified in Appendix~\ref{app:proof-assignment-induced-quotient}.

\subsection{Support-Realized Gaussian Prototypes}
\label{sec:prototypes}

We now instantiate the assignment structure above with the metric realization
used in our experiments. The purpose of this stage is to learn a geometry in
which evidence supporting the same persistent identity is represented near a
shared dataset-relative prototype, while evidence opposing that identity can
be separated from it. This follows the support-based metric-learning view of
Prototypical Networks~\citep[Section~2.2]{snell2017prototypical} and adopts
precision-weighted prototype construction from Gaussian Prototypical
Networks~\citep[Sections~3.1--3.3]{fort2017gaussian}; latent equivalence
learning itself does not require this particular realization. Let $m,p\in\mathbb N_{>0}$ denote the input-feature and learned metric
dimensions. We use a feature map $\Phi_D:\mathcal E(D)\longrightarrow\mathbb R^m$ to encode each evidence object and its available context as the numerical
input to the metric learner. This differs from $\omega_D^\pm(e,c)$: $\Phi_D(e)$ describes the evidence object itself, whereas $\omega_D^\pm(e,c)$ weight its association with a candidate identity. A learned encoder $\psi_\theta:
\mathbb R^m\longrightarrow\mathbb R^p\times\mathbb R_{>0}$ produces $(z_e,s_e)=\psi_\theta\!\left(\Phi_D(e)\right),$ where $z_e\in\mathbb R^p$ is the learned metric representation of $e$ and $s_e>0$ is its predicted scalar precision. For each $c\in\mathcal C$, let $\mu_c^0\in\mathbb R^p$ and
$s_c^0\in\mathbb R_{>0}$ denote its learned prior center and scalar precision, and define $\Lambda_c^0\Def s_c^0 I_p.$ We now use the supporting weights to form the dataset-relative realization of identity $c$. Following the precision-weighted support aggregation of Gaussian Prototypical Networks~\citep[Section~3.3, Eqs.~(5)--(6)]{fort2017gaussian}, evidence with greater association weight and predicted precision contributes more strongly to the realized prototype. Define
\[
    s^{\mathrm{sup}}_{c,D}
    \Def
    \sum_{e\in\mathcal{E}(D)}
    \omega^{+}_{D}(e,c)\,s_e.
\]
Whenever $s^{\mathrm{sup}}_{c,D}>0$, the corresponding support-derived center is
\[
    \mu^{\mathrm{sup}}_{c,D}
    \Def
    \frac{
        \sum_{e\in\mathcal{E}(D)}
        \omega^{+}_{D}(e,c)\,s_e z_e
    }{
        s^{\mathrm{sup}}_{c,D}
    }.
\]
Thus $s^{\mathrm{sup}}_{c,D}$ summarizes the total precision-weighted support
for identity $c$, while $\mu^{\mathrm{sup}}_{c,D}$ is the corresponding
support-derived center. Because the available dataset evidence may realize an identity with varying support, we retain the persistent prior as a fallback rather than replacing it unconditionally by the support-derived center. Let
$\alpha_{c,D}\in[0,1]$ control this interpolation when support is present:
$\alpha_{c,D}=1$ uses the dataset support entirely, while $\alpha_{c,D}=0$ retains the persistent prior. The active prototype center is
\[
    \mu_{c,D}
    =
    \begin{cases}
        \alpha_{c,D}\mu^{\mathrm{sup}}_{c,D}
        +(1-\alpha_{c,D})\mu_c^{0},
        & s^{\mathrm{sup}}_{c,D}>0,\\[4pt]
        \mu_c^{0},
        & s^{\mathrm{sup}}_{c,D}=0.
    \end{cases}
\]
We apply the same support--prior interpolation to precision. For the
scalar-precision realization used here, define
\[
    \Lambda^{\mathrm{sup}}_{c,D}
    \Def
    \bigl(\varepsilon+s^{\mathrm{sup}}_{c,D}\bigr)I_p,
    \quad \varepsilon>0,
\]
and let
\[
    \Lambda_{c,D}
    =
    \begin{cases}
        \alpha_{c,D}\Lambda^{\mathrm{sup}}_{c,D}
        +(1-\alpha_{c,D})\Lambda_c^{0},
        & s^{\mathrm{sup}}_{c,D}>0,\\[4pt]
        \Lambda_c^{0},
        & s^{\mathrm{sup}}_{c,D}=0.
    \end{cases}
\]
Since both component matrices are positive definite, so is $\Lambda_{c,D}$; the diagonal-precision case follows analogously by using coordinatewise positive precisions. The resulting dataset-relative realization of identity $c$ is
$P_{c,D}\Def(\mu_{c,D},\Lambda_{c,D}).$ The three objects introduced so far play different roles: $c$ is the shared identity label, $P_{c,D}$ is its learned geometric realization in dataset $D$, and $Q_{c,D}\subseteq\mathcal E(D)$ is the set of evidence objects assigned to that identity by the hard readout. To compare an embedded evidence object with a realized identity, define the precision-weighted prototype distance
\[
    d_{c,D}:
    \mathbb R^p
    \longrightarrow
    \mathbb R_{\geq 0},
    \quad
    d_{c,D}(z)
    \Def
    \sqrt{
        (z-\mu_{c,D})^\top
        \Lambda_{c,D}
        (z-\mu_{c,D})
    }.
\]
Following distance-based prototypical classification
\citep[Eq.~(2)]{snell2017prototypical} and its precision-weighted Gaussian
extension \citep[Algorithm~1]{fort2017gaussian}, we convert these competing
distances into the soft-membership profile introduced in Section~\ref{sec:fibers}:
\[
    [\rho_D(e)]_c
    \Def
    \frac{
        \exp\!\left(-d_{c,D}(z_e)/\tau\right)
    }{
        \sum_{c'\in\mathcal C}
        \exp\!\left(-d_{c',D}(z_e)/\tau\right)
    },
    \quad
    \tau>0.
\]
Thus $\rho_D(e)$ records relative membership across all realized identities,
while $\kappa_D(e)$ selects an identity attaining the maximal membership,
with ties resolved by $T$. We train this realization using the same supporting--opposing distinction that defines the evidence interface. Supporting evidence is encouraged to lie near the corresponding prototype,
\begin{equation*}
\mathcal{L}_{+}
=
\sum_{e\in\mathcal{E}(D)}
\sum_{c\in\mathcal{C}}
\omega^{+}_{D}(e,c)\,d_{c,D}^{\,2}(z_e),
\end{equation*}
whereas opposing evidence incurs a penalty whenever it lies within a prescribed squared-distance margin around the opposed identity realization,
\begin{equation*}
\mathcal{L}_{-}
=
\sum_{e\in\mathcal{E}(D)}
\sum_{c\in\mathcal{C}}
\omega^{-}_{D}(e,c)\,
\bigl[\xi_{-}-d_{c,D}^{\,2}(z_e)\bigr]_{+}^{2},
\quad
\xi_{-}\in\mathbb{R}_{>0},
\end{equation*}
where $[x]_{+}=\max\{x,0\}$. Both terms are nonnegative: supporting evidence creates attraction toward the corresponding realization, while opposing evidence creates separation from inappropriate realizations. For a training episode, let $\mathcal{T}_{\mathrm{qry}}$ denote its held-out evidence objects paired with target identities $(e,y)$, where $y\in\mathcal C$. Following the episodic classification objective of Prototypical Networks
\citep[Section~2.2]{snell2017prototypical}, define
\begin{equation*}
\mathcal{L}_{\mathrm{epi}}
\Def
-\sum_{(e,y)\in \mathcal{T}_{\mathrm{qry}}}
\log [\rho_D(e)]_y.
\end{equation*}
The complete training objective may then be written
\begin{equation*}
\mathcal{L}
=
\mathcal{L}_{\mathrm{epi}}
+
\lambda_{\mathrm{soft}}\mathcal{L}_{\mathrm{soft}}
+
\lambda_{+}\mathcal{L}_{+}
+
\lambda_{-}\mathcal{L}_{-}
+
\lambda_{\mathrm{unc}}\mathcal{L}_{\mathrm{unc}}
+
\lambda_{\mathrm{reg}}\mathcal{R},
\end{equation*}
with all coefficients nonnegative. Here $\mathcal L_{\mathrm{soft}}$
preserves consistency with soft or weak targets, while $\mathcal L_{\mathrm{unc}}$ and $\mathcal R$ collect realization-specific calibration and geometric regularization terms. The framework consumes application-specific evidence through $\omega_D^\pm$; their upstream construction remains separate from prototype aggregation. Having specified the learned geometry and the objective that shapes it, we can now characterize when its induced hard identity assignment is locally stable.

\paragraph{Local identity stability.}
For a fixed prototype bank, write
$h_D(z)\Def T(\operatorname*{arg\,min}_{c\in\mathcal C}d_{c,D}(z))$,
so that $\kappa_D(e)=h_D(z_e)$.
The following margin certificate specializes the standard margin--Lipschitz
argument to our prototype realization \citep{tsuzuku2018lipschitz}.

\begin{proposition}[Local assignment stability]
\label{prop:local-assignment-stability}
Fix $D$ and its prototype bank, and suppose $|\mathcal C|\geq2$.
Let $e\in\mathcal E(D)$ have unique nearest identity $c_*=h_D(z_e)$.
Set $L_{c,D}\Def\|\Lambda_{c,D}^{1/2}\|_{\mathrm{op}}$, where
$\|\cdot\|_{\mathrm{op}}$ is the Euclidean operator norm, and define
\begin{equation}
 r_D(e)\Def\min_{c\ne c_*}
 \frac{d_{c,D}(z_e)-d_{c_*,D}(z_e)}{L_{c,D}+L_{c_*,D}}>0.
 \label{eq:assignment-stability-radius}
\end{equation}
For every $\delta\in\mathbb R^p$ with $\|\delta\|_2<r_D(e)$,
$h_D(z_e+\delta)=c_*$. Consequently, any $e'\in\mathcal E(D)$ with
$\|z_{e'}-z_e\|_2<r_D(e)$ belongs to $Q_{c_*,D}$.
\end{proposition}
Appendix~\ref{app:assignment-stability} proves the result and
Appendix~\ref{app:signed-evidence-stability} relates its margin to the
supporting and opposing losses. The certificate concerns perturbations of
the embedding with the bank fixed; arbitrary schema changes or rebuilt
prototypes require additional control.

\subsection{Identity-Factorized Routing and Agentic Composition}
\label{sec:query-routing}

The evidence-side representation tells us how the current dataset realizes
the persistent identities; routing must determine which of those identities
a request requires. We therefore represent each query independently in a
second prototype space and then map its query-side memberships into
activations over $\mathcal C$. Let $\mathcal Q$ denote the space of admissible natural-language requests, and let $(\mathcal Z_Q,d_Q)$ be a learned query metric space, with $d_Q:\mathcal Z_Q\times\mathcal Z_Q\longrightarrow\mathbb R_{\geq 0}.$ A learned query encoder $g_\phi:\mathcal Q\longrightarrow
\mathcal Z_Q$ maps a request $q$ to $u_q\Def g_\phi(q)$. Let $\mathcal K$ index a finite, nonempty set of learned query prototypes
$\{\eta_k\}_{k\in\mathcal K}\subseteq\mathcal Z_Q$, where each prototype
represents a recurring evidential requirement. For $\tau_Q>0$, define
\[
 [\beta(q)]_k
 \Def
 \frac{\exp(-d_Q(u_q,\eta_k)/\tau_Q)}
 {\sum_{\ell\in\mathcal K}\exp(-d_Q(u_q,\eta_\ell)/\tau_Q)}.
\]
Thus $\beta:\mathcal Q\longrightarrow\Delta(\mathcal K)$ assigns each request a soft profile over recurring query requirements. The query profile must next be expressed in the identity coordinates used by the evidence-side representation. Let $M\in\mathbb R^{|\mathcal K|\times|\mathcal C|}$ be a learned compatibility matrix, where $M_{kc}$ represents the learned compatibility between query prototype $k$ and persistent identity $c$. Define $\Gamma_M:\Delta(\mathcal K)
\longrightarrow\Delta(\mathcal C),$ where $\Gamma_M(b)\Def\operatorname{softmax}(M^\top b),$ and let
\begin{equation}
    \gamma(q)
    \Def
    \Gamma_M(\beta(q))
    =
    \operatorname{softmax}\!\left(M^\top\beta(q)\right).
    \label{eq:query-identity-map}
\end{equation}
Thus $\gamma:\mathcal Q\longrightarrow\Delta(\mathcal C)$ assigns each request an activation profile over persistent identities. The coordinates of $\rho_D(e)$ and $\gamma(q)$ share the index set $\mathcal C$, but they have different meanings: $\rho_D(e)$ represents evidence membership, whereas $\gamma(q)$ represents query demand. To reuse the learned dataset representation across requests, let $\mathsf S_D$ denote the finite, source-preserving evidence state produced for dataset $D$. Its records retain the learned embeddings and memberships together with supporting and opposing evidence, prototype diagnostics, and provenance required for later retrieval.
A query should expose only a bounded subset of this state. Let $\operatorname{cost}:2^{\mathsf S_D}\longrightarrow\mathbb R_{\geq 0}$ measure the context cost of a selected evidence view, and for a budget
$B>0$ define
\[
    \mathfrak V_B(D)
    \Def
    \left\{
        V\subseteq\mathsf S_D:
        \operatorname{cost}(V)\leq B
    \right\}.
\]
We model query-conditioned materialization by $\mathsf C_B:2^{\mathsf S_D}\times\Delta(\mathcal C)\longrightarrow\mathfrak V_B(D),$ and, for fixed $D$, define
\[
    \mathsf F_{D,B}:
    \Delta(\mathcal C)
    \longrightarrow
    \mathfrak V_B(D),
    \qquad
    \mathsf F_{D,B}(\pi)
    \Def
    \mathsf C_B(\mathsf S_D,\pi).
\]
The bounded evidence view presented for request $q$ is therefore
\begin{equation}
    \mathcal V_B(D,q)
    \Def
    \mathsf F_{D,B}(\gamma(q))
    =
    \bigl(
        \mathsf F_{D,B}\circ\Gamma_M\circ\beta
    \bigr)(q),
    \quad
    \mathcal V_B(D,q)\in\mathfrak V_B(D).
    \label{eq:identity-factorized-routing}
\end{equation}
The materialization retains query-relevant supporting and opposing evidence,
competing realizations, and provenance; source references permit scoped
recovery of evidence omitted by the budget.

\paragraph{What the factorization preserves.}
Fix the identity inventory and query-side model, with no additional
dataset-dependent inputs. Then $\gamma(q)$ depends only on the request:
changing the dataset replaces $\mathsf F_{D,B}$ but leaves the query
activation profile unchanged. For fixed $D$, budget $B$, and deterministic
materialization,
\[
    \beta(q)=\beta(q')
    \;\Longrightarrow\;
    \gamma(q)=\gamma(q')
    \;\Longrightarrow\;
    \mathcal V_B(D,q)=\mathcal V_B(D,q').
\]
Thus the factorization separates reusable query requirements from their
dataset-specific evidence realization. It does not assert that evidence views
are identical across datasets, that the requested evidence exists, or that the
view determines the final answer. Appendix~\ref{app:identity-factorized-routing}
gives the corresponding formal statement and explains why a shared hard
query-prototype label alone is insufficient. The bounded evidence view is then interpreted against the original request by a single LLM Arbiter. Let $\mathsf R_{\mathrm{Arb}}$ denote the space of typed, source-preserving Arbiter records. Define
\[
    \mathcal A_B(D,q)
    \Def
    \mathsf{Arb}\!\left(q,\mathcal V_B(D,q)\right)
    =
    (a_1,\ldots,a_{m_B}),
    \qquad
    a_j\in\mathsf R_{\mathrm{Arb}},
\]
where $m_B<\infty$. Each record preserves its source reference and specifies
evidence or execution requirements for the downstream agent. The Arbiter
operates only on the bounded view; it does not search the full static state
and therefore does not replace the learned routing stage. Let $\bigodot$ denote source-preserving composition of these records, with each component remaining individually accessible. The final answer is
\[
    \widehat y
    =
    \mathsf{Agent}\!\left(
        q,\,
        \bigodot_{j=1}^{m_B} a_j
        \mathsf{DuckDB}(D),\,
        \mathsf{Calc}
    \right).
\]
The agent uses the routed records to retrieve required values, execute
computations, and verify its response. Because both the Arbiter and the final
agent retain the original request $q$, identical evidence views need not
produce identical answers for different queries.

\section{Experiments}
\label{sec:experiments}

\subsection{Experimental Setup}

\paragraph{Benchmark and evaluation protocol.}
We evaluate on the Data Agent Benchmark (DAB), an end-to-end benchmark for data agents spanning 54 queries across 12 datasets, nine application domains, and four database systems \cite{ma2026dab}. DAB includes heterogeneous structured-data tasks requiring operations such as multi-table reasoning,
irregular join resolution, aggregation, temporal reasoning, and
unstructured-text interpretation. One complete trial consists of one independent attempt at each of the 54 benchmark queries. We evaluate each primary Permute system over \(n=5\) complete trials, yielding \(54\times5=270\) query attempts per system. Each query is evaluated by a query-specific validator. Let $\mathcal D$ denote the set of benchmark datasets, let
$\mathcal Q_d\subseteq\mathcal Q$ denote the benchmark queries associated with dataset $d\in\mathcal D$, and let $r_{dqi}\in\{0,1\}$ indicate whether trial $i$ succeeds on query $q\in\mathcal Q_d$ in dataset $d$. The empirical success rate of a query is $\widehat{p}_{dq}=\frac{1}{n}\sum_{i=1}^{n} r_{dqi}.$ DAB's primary metric is dataset-macro stratified Pass@1,
\[
    \operatorname{Pass@1}_{\mathrm{strat}}
    =
    \frac{1}{|\mathcal D|}
    \sum_{d\in\mathcal D}
    \frac{1}{|\mathcal Q_d|}
    \sum_{q\in\mathcal Q_d}
    \widehat p_{dq}.
\]
Thus, queries are first averaged within each dataset and each dataset receives equal weight. The small markers in Figure~\ref{fig:dab-results} show the corresponding
dataset-macro score for each complete trial, while the large markers show system means. We additionally report raw query-attempt accuracy,
\[
    A^{\mathrm{micro}}
    =
    \frac{
        \displaystyle
        \sum_{d\in\mathcal D}
        \sum_{q\in\mathcal Q_d}
        \sum_{i=1}^{n} r_{dqi}
    }{
        \displaystyle
        n\sum_{d\in\mathcal D}|\mathcal Q_d|
    }.
\]
which weights each of the \(270\) query attempts equally rather than weighting datasets equally.

\paragraph{Systems and ablations.}
We compare three primary systems: the DAB Claude Opus~4.6 reference agent
using the standard benchmark scaffold; Core, using the agentic harness
with an Opus~5 reasoner but without the proposed learned equivalence
representations; and our architecture, which adds the prototype
realizations and query-conditioned evidence construction of
Section~\ref{sec:method}. While DAB's Claude Opus~4.6 reference agent utilizes a weaker model, it serves as a baseline to highlight the stark differences performance can make, even when using a model with a relevant level of capability. The Core comparison retains the stronger agentic setting while omitting the learned representation. Furthermore, the broader DAB leaderboard provides additional system-level context for these comparisons.

\subsection{Results and Ablations}

\paragraph{Overall DAB performance.}
The DAB Claude Opus~4.6 reference agent obtains \(55.51\%\) dataset-macro
stratified Pass@1, while Core reaches \(84.13\%\). Our architecture reaches
\textbf{\(94.67\%\)}, improving on Core by \textbf{\(10.54\)} percentage
points and on the reference agent by \textbf{\(39.16\)} points. At the raw
query-attempt level, our architecture succeeds on \textbf{\(258/270\)} attempts (\textbf{\(95.56\%\)}), compared with \(236/270\) (\(87.41\%\)) for Core,
a net gain of \(22\) successes.
\begin{figure}[H]
\centering
\includegraphics[width=0.485\linewidth]{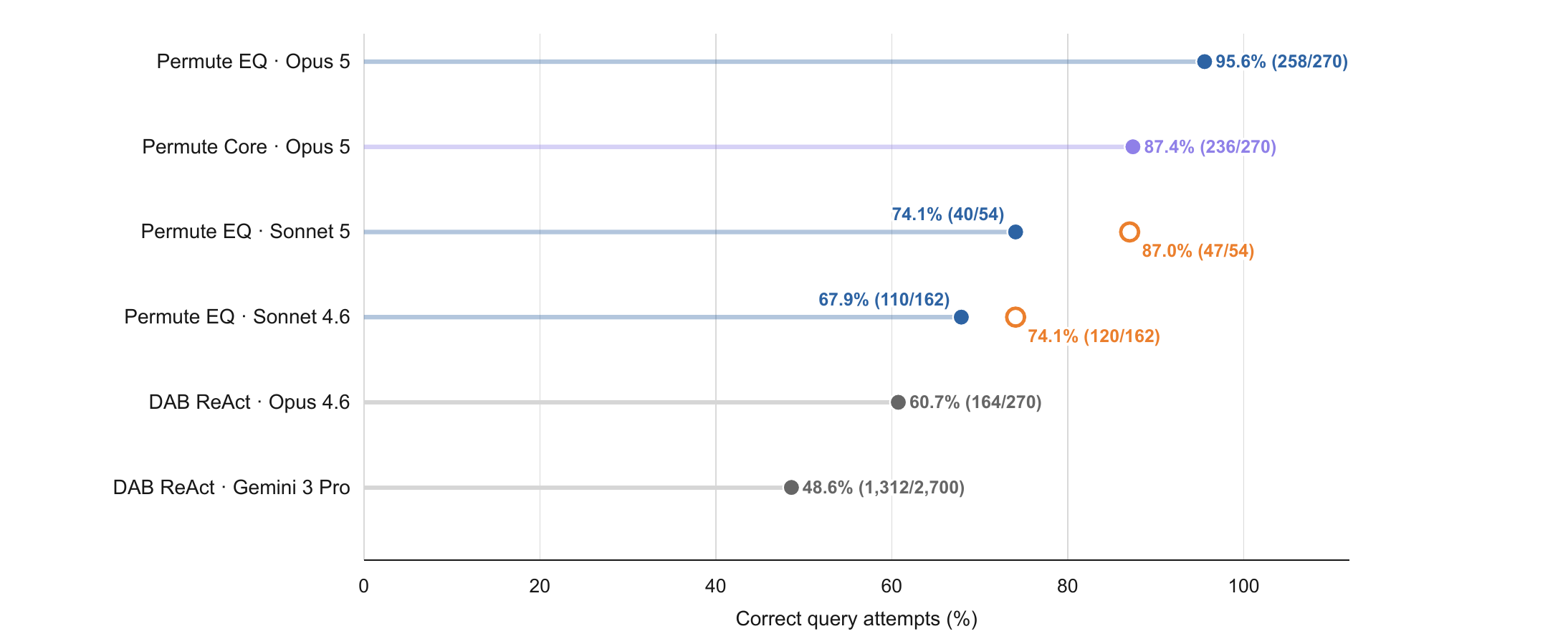}
\hfill
\includegraphics[width=0.485\linewidth]{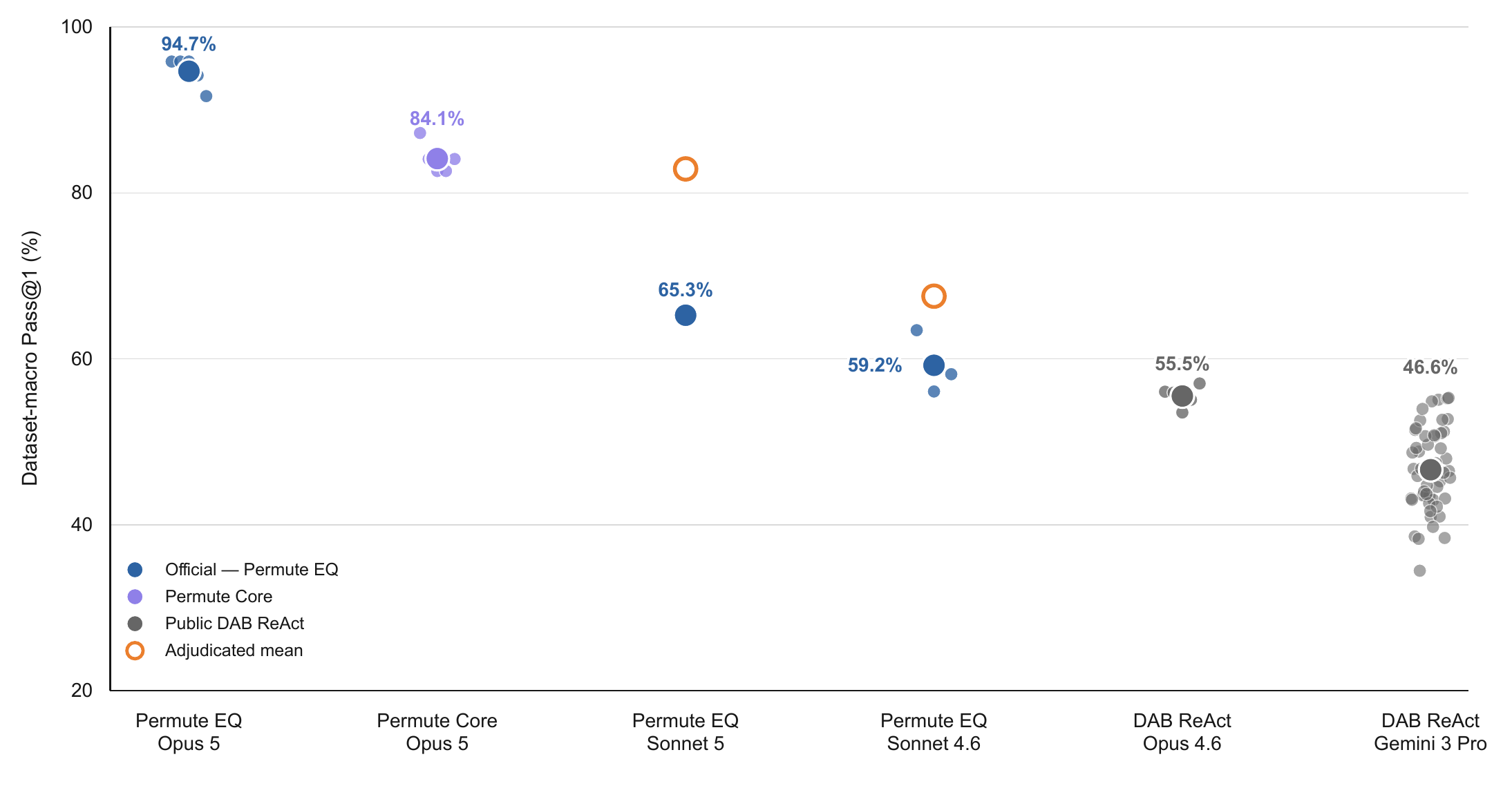}
\caption{DAB performance by raw query-attempt accuracy (left) and
dataset-macro stratified Pass@1 (right). Our proposed architecture is labeled
\textsc{Permute EQ} in the benchmark submission and figures, while
\textsc{Permute Core} denotes the corresponding implementation using the
conventional data-agent approach without the learned equivalence
representation and identity-factorized routing.}
\label{fig:dab-results}
\end{figure}
Relative to Core, the gain is therefore obtained on top of an already strong
agentic system; the distinguishing contribution is the learned
identity-based representation and query-conditioned evidence routing upstream
of the downstream reasoner. Comprehensive breakdowns of the query attempts and some additional relevant tracking metrics are detailed in the accompanying appendices.

\section{Conclusion}
\label{sec:conclusion}

Our results indicate that data-agent performance depends materially on how
evidence is organized before it reaches the language model. In heterogeneous
data environments, additional search and agentic decomposition can expand
exploration without actually resolving the underlying problem of which evidence should be compared, preserved, or exposed to the reasoner. Latent equivalence
learning moves this organization upstream by representing evidence through
persistent task-relevant identities and materializing a bounded,
query-conditioned view of their supporting, opposing, competing, and
provenance-bearing evidence.

Taken together, these results support the practical value of separating
reusable query requirements from dataset-specific evidence realization.
DAB stresses this separation across heterogeneous schemas and query forms,
including join resolution, aggregation, temporal reasoning, and
unstructured-text interpretation, so the observed improvement is not confined
to a single operation or data layout. The formal analysis gives a
complementary view of the same behavior: under the stated query-side
assumptions, the requirements induced by a request can remain fixed while
their materialized evidence changes with the dataset, while within a fixed
realization the hard evidence assignments are stable inside an explicit local
margin. These are distinct forms of controlled reuse---across environments on
the routing side and under perturbation on the representation side---and
together explain how learned structure can be reused without requiring the
downstream agent to reconstruct it for every request. The same analysis also
makes clear where that claim stops.

That scope is finite. The framework assumes that a bounded inventory of
persistent identities provides meaningful coverage of the deployment domain
and that the available evidence is sufficient to realize and distinguish
them. Requests that depend on identities outside this inventory,
unsupported relationships, or distinctions absent from the evidence may
therefore fall outside the modeled structure. Likewise, the local stability
certificate is conditional on a fixed prototype bank; substantial changes to
the learned realization require separate control. Performance consequently
remains limited by identity coverage, evidence quality, routing, and the
materialization budget.

Within these bounds, the broader implication is that evidence organization
need not remain an implicit burden of the language model. Across the
heterogeneous DAB tasks studied here, our results show that learning reusable
structure upstream can improve data-agent performance while preserving a
clear separation between representation and query-specific reasoning. This
supports a broader design principle for structured-data agents: learn the
organization that can be reused, and reserve language-model reasoning for the
semantic interpretation and execution that remain specific to the request.


\subsection{AI use statement}

AI tools, specifically ChatGPT 5.6 SOL Pro, were used to assist with the following:
\begin{itemize}
    \item drafting and refining prose across sections;
    \item improving the formalization, consistency, and canonicalization of notation;
    \item assisting with code development, debugging, and ablation design; and
    \item assisting with the presentation of graphs, figures, and diagrams.
\end{itemize}
All mathematical proofs, implementation details, algorithmic constructions, code design, diagram design, and architectural decisions are the authors' own, although AI did help confirm any core questions surrounding some of these topics (in a limited scope). For any figures or images produced with AI assistance, the underlying content, structure, and intended presentation were designed by the authors prior to using AI; AI was used only to help render or refine those designs in a visually polished form. Similarly, AI was used extensively for proofreading and stylistic refinement of the manuscript, after which the authors independently reviewed, revised, and incorporated the resulting prose into the paper in their own style.


\subsection{Reproducibility statement}
\label{sec:reproducibility}

We provide the architectural definitions, training objectives, and inference procedure necessary to characterize the proposed method in the main text, with complete mathematical proofs and additional derivations provided in the appendix. The experimental section specifies the benchmark protocol, evaluation metrics, comparison systems, and repeated-trial methodology used to obtain the reported results. Certain implementation details, including hyperparameter configurations, are not disclosed, and the source code is not released because the system is currently proprietary.

All Data Agent Benchmark (DAB) evaluations strictly followed the official benchmark rules (cf. \href{https://ucbepic.github.io/DataAgentBench/}
{Data Agent Benchmark}; \citealp{ma2026dab}), and the reported results were externally validated. The validation protocol prohibited access to gold-standard answers and prohibited benchmark-specific information or leading cues from being incorporated into the system inputs or prompts. Consequently, the system was required to recover the information necessary to answer each query through its learned representations, data-agent execution, and subsequent arbiter reasoning rather than through access to benchmark answers or benchmark-specific prompt tuning.

\paragraph{Compute resources.}

All training and evaluation were conducted locally on two consumer-grade workstations; no external compute clusters or cloud-based training infrastructure were used. The first workstation was equipped with an NVIDIA GeForce RTX 3070 Ti with 8\,GB of GPU memory, an Intel Core i9-9900K CPU, and 32\,GB of system memory. The second workstation was equipped with an NVIDIA GeForce RTX 4060 Ti with 16\,GB of GPU memory, an Intel Core i9-9900K CPU, and 128\,GB of system memory. Both systems had at least 1\,TB of local storage available throughout the experiments.



\bibliographystyle{iclr2027_conference}
\bibliography{./references}


\appendix

\section{Additional Graphs}

\subsection{Per-Dataset Performance}
\label{app:dab-dataset-performance}

Figure~\ref{fig:dab-dataset-accuracy} disaggregates the aggregate DAB
results in Section~\ref{sec:experiments} by dataset. Each cell reports
dataset-level Pass@1 for the corresponding system configuration, making
visible where aggregate differences arise rather than treating the benchmark
as a single homogeneous task family.

\begin{figure}[H]
    \centering
    \includegraphics[width=0.82\linewidth]{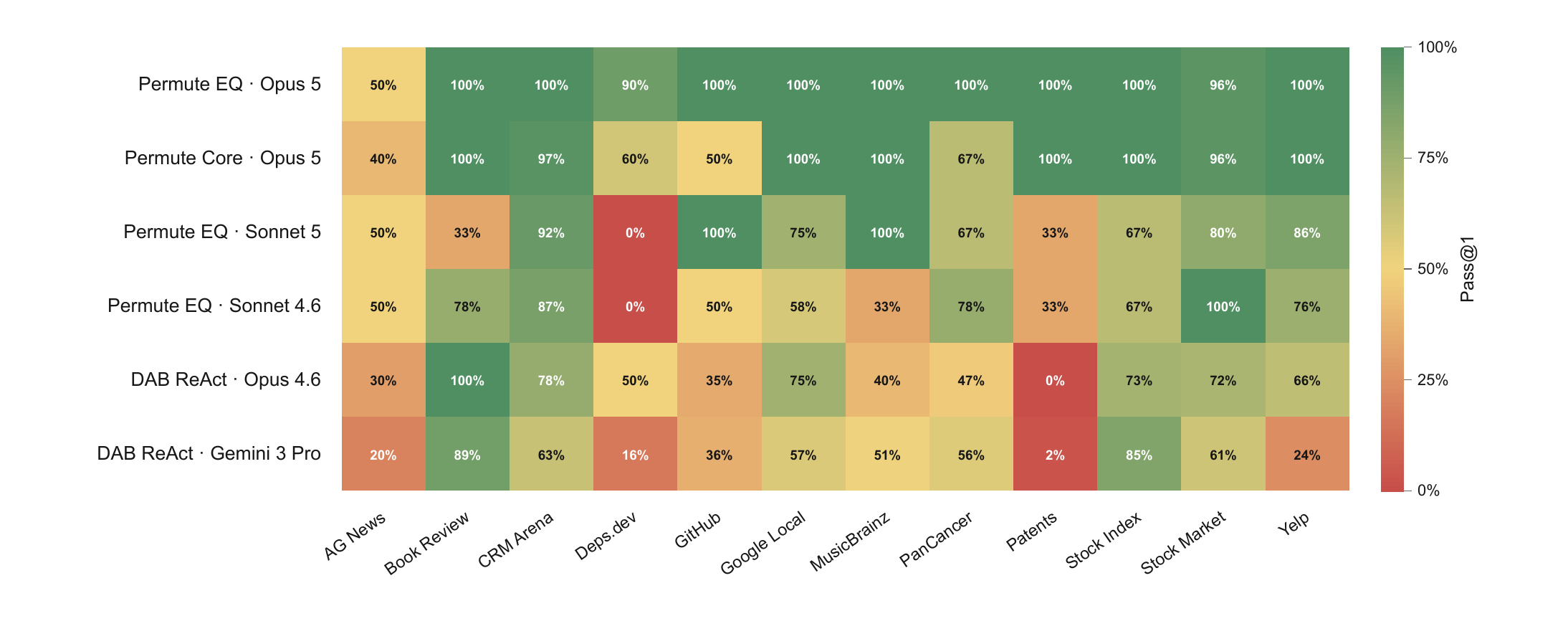}
    \caption{\textbf{Per-dataset DAB Pass@1 across evaluated configurations.}
    Each cell reports the Pass@1 percentage for one of the 12 benchmark
    datasets. The full Opus~5 realization attains perfect Pass@1 on nine
    datasets and exceeds the corresponding Core configuration on AG News,
    CRM Arena, Deps.dev, GitHub, and PanCancer, while matching it on the
    remaining seven datasets.}
    \label{fig:dab-dataset-accuracy}
\end{figure}

The largest gaps relative to Core occur on GitHub ($50\%\!\rightarrow\!100\%$), PanCancer ($67\%\!\rightarrow\!100\%$), and Deps.dev ($60\%\!\rightarrow\!90\%$). These differences are descriptive; the
aggregate evaluation remains the primary comparison reported in the main text. 

\subsection{Cost, resources, and consistency}
\label{app:dab-cost-resources-consistency}

The following diagnostics characterize the operating point used in our
benchmark runs rather than a token-optimized deployment. Permute EQ was
configured for conservative finalization: after locating a candidate answer,
the system could continue to retrieve corroborating evidence, challenge
competing interpretations, recompute intermediate results, and pass the
answer through explicit arbitration and criticism before returning it. In
several manually inspected traces, the eventual correct answer appeared
within approximately \(3\)--\(5\) tool calls, after which the system used on
the order of \(15\) additional calls to verify the result. This trace
observation is qualitative rather than a benchmark-wide average, but it
illustrates that a substantial portion of the observed runtime can arise
after answer discovery.

We retained these guardrails deliberately. Our objective in this evaluation
was to measure the complete system under an accuracy- and safety-oriented
verification policy, not to minimize token use through aggressive early
stopping. We report the resulting resource consumption because it exposes
the cost of that choice and identifies where further optimization is
possible. The current measurements should therefore be read as a
high-reliability operating point rather than as a claim of token optimality.

\paragraph{Quality and monetary cost.}
Figure~\ref{fig:dab-cost} compares dataset-macro quality with the mean cost
of one complete trial. The horizontal axis shows mean cost per complete trial, and movement toward the upper-left corresponds to greater capability at lower cost.
Across the EQ configurations, the figure shows the practical effect of
downstream reasoner choice. Sonnet~5 reaches \(82.9\%\) stratified Pass@1
at a mean cost of \(\$41.10\), nearly matching the \(84.1\%\) Core result
while remaining below Core's reported mean cost of approximately \(\$56\).
Permute EQ with Opus~5 reaches the highest measured quality,
\(94.7\%\), at a higher mean cost of \(\$82.82\).

Only the Core-versus-EQ Opus~5 comparison holds the downstream reasoner and
agentic setting fixed. Comparisons among the Opus and Sonnet configurations
instead illustrate different quality--cost operating points. The reported
points have not been jointly optimized over reasoner choice, verification
depth, stopping policy, or token budget, and Core's cost remains approximate
because its accepted artifacts omit the corresponding cost fields.

\begin{figure}[H]
  \centering
  \includegraphics[width=0.6\linewidth]{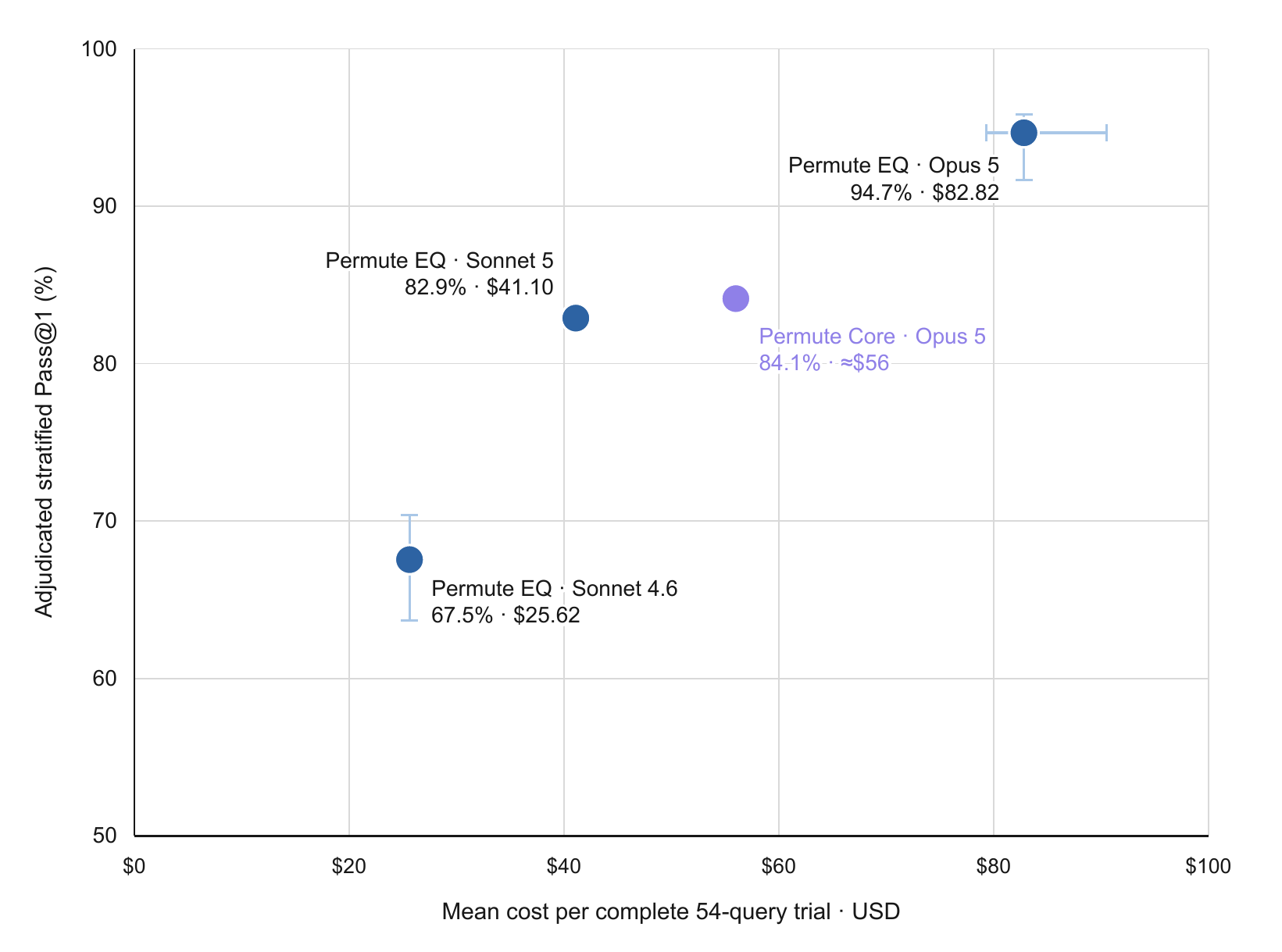}
  \caption{
    Dataset-macro DAB quality versus mean cost per complete 54-query trial
    for the evaluated configurations. Permute EQ points use adjudicated
    quality and measured scorecard cost. Core's accepted artifacts omit cost
    fields, so its point uses the reported mean of approximately \(\$56\)
    per run for the final max-reasoning configuration. The displayed points
    were not optimized jointly for token efficiency and verification depth.
  }
  \label{fig:dab-cost}
\end{figure}

\paragraph{Cumulative token and tool usage.}
Figure~\ref{fig:dab-resources} reports mean cumulative token and tool-call
usage over a complete 54-query trial. These totals include not only the
calls used to locate and compute an answer, but also subsequent
corroboration, contradiction checks, recomputation, arbitration, and critic
validation. They therefore measure the full conservative execution policy
used in the reported runs.

This distinction is also important for interpreting the bounded evidence
view. The bound applies to the evidence materialized for an individual
reasoning step; it does not bound the cumulative number of scoped retrieval,
execution, or verification steps performed across an entire trial. A system
may therefore maintain bounded prompt-facing evidence while still
accumulating substantial token and tool usage through repeated verification.
The larger EQ totals should be interpreted in this context rather than as
evidence that query conditioning failed to constrain the evidence exposed
at any one step.

\begin{figure}[H]
  \centering
  \includegraphics[width=0.5\linewidth]{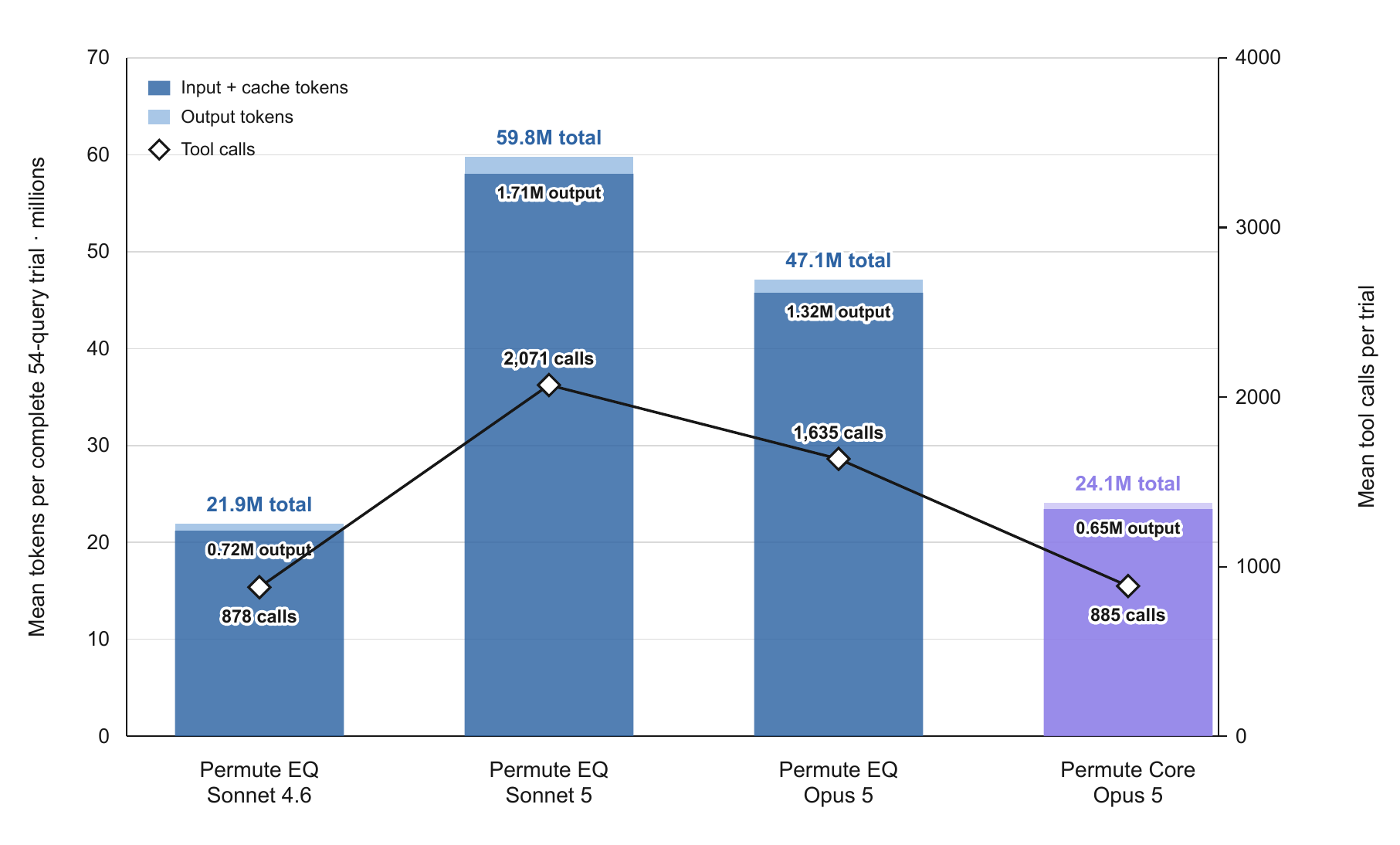}
  \caption{
    Mean token and tool-call usage per complete 54-query trial under the
    evaluated accuracy-first configurations. Each bar separates input and
    cache tokens from output tokens, while the line reports mean tool calls.
    The totals include conservative post-answer verification. Core's five
    reported usage summaries average \(23.455\) million input and cache
    tokens and \(0.655\) million output tokens, or \(24.110\) million total
    tokens.
  }
  \label{fig:dab-resources}
\end{figure}

Several of the current guardrails could in principle be reduced through confidence-dependent stopping, adaptive verification depth, fewer repeated challenge paths, or earlier termination once independent checks agree. Those alternatives were not evaluated here. The present results therefore identify an unoptimized point on a broader tradeoff among offline training
duration, downstream reasoner capability, online token and tool use, monetary cost, and final-answer reliability. We hypothesize that more careful tuning of this frontier could retain comparable benchmark performance while materially reducing online cost, but establishing such a Pareto-efficient configuration requires a dedicated ablation.

\begin{figure}[H]
  \centering
  \includegraphics[width=0.5\linewidth]{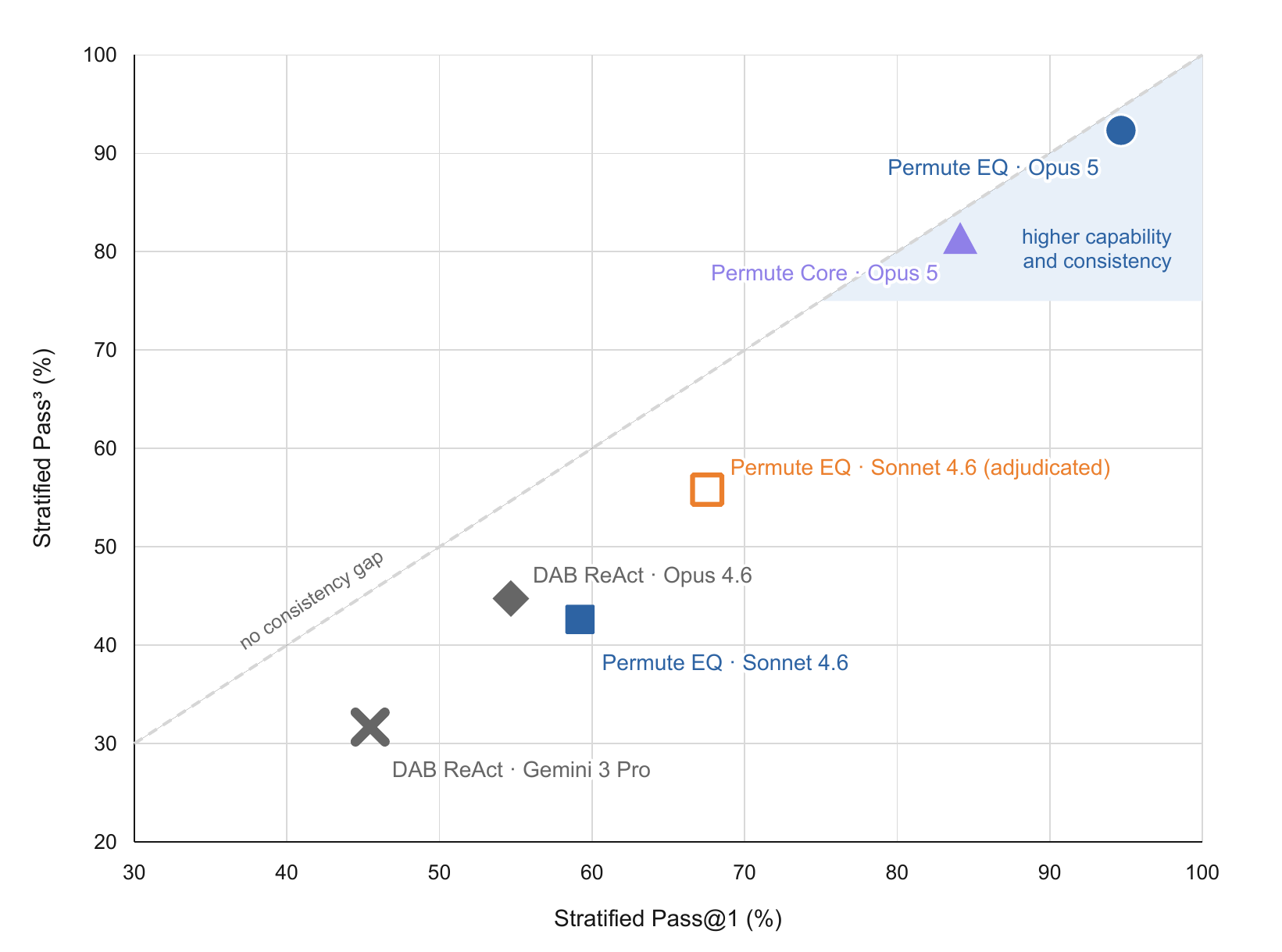}
  \caption{
    Dataset-macro capability versus repeated-trial consistency.
    Pass\(^{3}\) requires a query to pass in all three sampled trials and is
    distinct from Pass@3. The diagonal denotes no consistency gap; points
    closer to the upper-right combine stronger average performance with
    greater trial-to-trial stability. For the 50-run Gemini result, the
    displayed consistency score averages over every three-run subset.
  }
  \label{fig:dab-consistency}
\end{figure}

\paragraph{Capability and repeated-trial consistency.}
Figure~\ref{fig:dab-consistency} compares ordinary stratified Pass@1 with
Pass\(^{3}\). Pass\(^{3}\) is not Pass@3: a query contributes a success only
when it passes in all three sampled trials, after which the query-level
indicators are aggregated using the same dataset-macro structure as
Pass@1. The horizontal axis therefore measures average capability, while
the vertical axis measures how much of that capability is retained
consistently across repeated trials.

Since a query that passes all three trials must also pass an individual
trial, Pass\(^{3}\) cannot exceed Pass@1. The diagonal represents no
consistency gap, while vertical distance below it reflects run-to-run
fragility. Permute EQ with Opus~5 occupies the upper-right region and lies
close to the diagonal, combining the highest average capability with strong
three-run consistency. This behavior is consistent with the system's
conservative finalization policy, although the present experiments do not
isolate how much of the consistency is attributable to verification depth,
the learned equivalence representation, or their interaction. A
verification-depth ablation would be required to make that causal
attribution.

\section{Assignment Structure, Routing, and Stability}
\label{app:proofs}

We separate elementary properties of the assignment maps from the routing
factorization and conditional stability guarantees. Unless a change of
realization is explicitly considered, the trained parameters and the
prototype bank for $D$ are fixed. All weights and model outputs below are
finite. The statements characterize the disclosed mathematical formulation;
they do not assume access to the mechanisms constructing $\omega_D^\pm$.

\subsection{Assignment Fibers and Soft Membership}
\label{app:proof-assignment-induced-quotient}

\paragraph{Equivalence and quotient.}
Because $\mathcal C$ is finite and nonempty, the maximizer set of every
$\rho_D(e)$ is nonempty. The fixed selector $T$ makes $\kappa_D$ single-valued.
Reflexivity, symmetry, and transitivity of equality therefore imply the same
three properties for $e\sim_D e'\iff\kappa_D(e)=\kappa_D(e')$.
For every $e\in\mathcal E(D)$,
\[
 [e]_D
 =\{e':\kappa_D(e')=\kappa_D(e)\}
 =Q_{\kappa_D(e),D}.
\]
Every evidence object belongs to its own assignment fiber, and two fibers
intersect only if their identity labels agree. The nonempty fibers consequently
partition $\mathcal E(D)$. The map
$[e]_D\mapsto\kappa_D(e)$ is well-defined, injective, and surjective onto
$\kappa_D(\mathcal E(D))$, proving the stated quotient correspondence.
These facts hold for any single-valued assignment map; no novelty is claimed
for this construction.

\paragraph{Prototype decision cells.}
\label{app:geometric-realization-fibers}
Let $z_D(e)\Def z_e$ name the embedding component of
$\psi_\theta(\Phi_D(e))$. The prototype construction is well-defined:
$s^{\mathrm{sup}}_{c,D}\geq0$, its positive branch has a nonzero denominator,
and its zero-support branch uses the prior. The prior and support precision
matrices are positive definite, as are their convex combinations.
Consequently every $d_{c,D}(z)$ is finite and nonnegative. Since $\tau>0$,
maximizing the distance softmax is equivalent to minimizing $d_{c,D}$,
including all ties. Hence
\[
 \kappa_D=h_D\circ z_D,\qquad
 V_{c,D}\Def h_D^{-1}(\{c\}),\qquad
 Q_{c,D}=z_D^{-1}(V_{c,D}).
\]
The final equality follows because
$\kappa_D(e)=c\iff h_D(z_D(e))=c\iff z_D(e)\in V_{c,D}$.
Thus $P_{c,D}$ is a prototype realization, $V_{c,D}$ is a decision cell
obtained by comparing all prototypes, and $Q_{c,D}$ is its evidence-space
preimage. Neither the identity label nor an individual prototype is the
set-valued equivalence class.

\begin{proposition}[Membership pseudometric]
\label{prop:membership-pseudometric}
\label{app:proof-membership-pseudometric}
The function $d_D^{\mathrm{mem}}(e,e')=\|\rho_D(e)-\rho_D(e')\|_2$
is a pseudometric. Its zero-distance equivalence classes refine the
hard-assignment fibers, and the refinement can be strict.
\end{proposition}
\begin{proof}
Nonnegativity, zero self-distance, and symmetry follow from the Euclidean
norm. Its triangle inequality gives
\[
 \|\rho_D(e)-\rho_D(e'')\|_2
 \leq\|\rho_D(e)-\rho_D(e')\|_2+
       \|\rho_D(e')-\rho_D(e'')\|_2.
\]
Zero distance is equivalent to equality of the membership profiles, which
implies equal hard readouts under the same $T$. Conversely, in a two-identity
example, $(3/4,1/4)$ and $(2/3,1/3)$ have the same unique maximizer but differ
by Euclidean distance $\sqrt{2}/12$. Thus hard equivalence need not imply
zero membership distance. This also shows why the membership pseudometric
need not descend to a metric on the coarser hard quotient.
\end{proof}

\paragraph{Remaining well-definedness checks.}
\label{app:additional-well-posedness}
The finite softmax normalizers defining $\rho_D$, $\beta$, and $\gamma$
are strictly positive, so the resulting vectors belong to their respective
simplices. Each summand of $\mathcal L_+$ and $\mathcal L_-$ is nonnegative.
An opposing-evidence term vanishes when its squared distance reaches
$\xi_-$; it is strictly positive below that margin only when its weight
$\omega_D^-(e,c)$ is positive. None of these observations guarantees that
optimization attains every desired separation.

\subsection{Identity-Factorized Routing}
\label{app:identity-factorized-routing}

Fix the shared identity inventory, query projector $g_\phi$, query metric,
query prototypes, and compatibility matrix $M$. As in Section~\ref{sec:query-routing}, these query-side objects receive no
additional dataset-dependent input. For fixed $D,B$, take conditioning to
be deterministic, or condition on a fixed realization of its randomness.

\begin{proposition}[Identity-factorized evidence selection]
\label{prop:identity-factorized-routing}
With $\Gamma_M(b)=\operatorname{softmax}(M^\top b)$ and
$\mathsf F_{D,B}(\pi)=\mathsf C_B(\mathsf S_D,\pi)$,
\[
 \mathcal V_B(D,q)
 =\mathsf F_{D,B}\bigl(\Gamma_M(\beta(q))\bigr).
\]
For any $q,q'\in\mathcal Q$,
\[
 \beta(q)=\beta(q')
 \ \Longrightarrow\ \gamma(q)=\gamma(q')
 \ \Longrightarrow\ \mathcal V_B(D,q)=\mathcal V_B(D,q').
\]
For the same query and fixed query-side model, replacing $D$ by another
environment $D'$ over the same identity inventory leaves $\gamma(q)$
unchanged; only the dataset-specific materialization map is replaced.
\end{proposition}
\begin{proof}
Substitution of $\gamma=\Gamma_M\circ\beta$ into the definition of
$\mathcal V_B$ gives the factorization. Equal $\beta$ profiles have equal
images under $\Gamma_M$; equal activation profiles then have equal images
under the fixed map $\mathsf F_{D,B}$. Neither $\beta(q)$ nor $M$ depends on
$D$ under the stated assumptions, so changing $\mathsf S_D$ affects
$\mathsf F_{D,B}$ but not $\gamma(q)$.
\end{proof}

This is an architectural separation property, not a statistical sufficiency
claim for the answer. The requested evidence may be absent in $D'$, the
materialized views may differ, and the Arbiter and agent still receive the
original query. If the query encoder, compatibility map, identity inventory,
or query-side inputs change with the dataset, the dataset-invariance
conclusion does not follow.

\paragraph{Why the query-side hard class is not enough.}
A deterministic hard readout of $\beta$ could itself define query classes,
but routing uses the full profile. For example, with
$\mathcal K=\mathcal C=\{1,2\}$ and $M=I_2$, the profiles
$b=(3/4,1/4)^\top$ and $b'=(2/3,1/3)^\top$ share their hard maximizer, while
$\Gamma_M(b)\ne\Gamma_M(b')$: their output odds are respectively
$\exp(1/2)$ and $\exp(1/3)$. Thus a routing map on hard query classes would
require an additional constancy condition that the present framework does
not impose. The compatibility matrix connects two soft representations;
it does not declare query requirements equivalent to evidence identities.

\subsection{Local Stability of Prototype Assignments}
\label{app:assignment-stability}

The proof applies the standard margin--Lipschitz certification principle
\citep{tsuzuku2018lipschitz} to the distances in
Section~\ref{sec:prototypes}. The square root
$\Lambda_{c,D}^{1/2}$ is the unique symmetric positive-definite square root;
$\|A\|_{\mathrm{op}}=\sup_{\|x\|_2=1}\|Ax\|_2$ is the Euclidean operator
norm. In particular,
$L_{c,D}=\sqrt{\lambda_{\max}(\Lambda_{c,D})}>0$.

\begin{proof}[Proof of Proposition~\ref{prop:local-assignment-stability}]
Write $A_{c,D}=\Lambda_{c,D}^{1/2}$, so that
$d_{c,D}(z)=\|A_{c,D}(z-\mu_{c,D})\|_2$.
For any $z,z'\in\mathbb R^p$, the reverse triangle inequality gives
\begin{align*}
 |d_{c,D}(z')-d_{c,D}(z)|
 &\leq\|A_{c,D}(z'-z)\|_2\\
 &\leq L_{c,D}\|z'-z\|_2.
\end{align*}
Fix $e$ and its unique winning identity $c_*$. For each competitor
$c\ne c_*$, put
$\Delta_{c,D}(e)=d_{c,D}(z_e)-d_{c_*,D}(z_e)>0$.
Applying the preceding inequality to the competitor and winner separately,
\begin{align*}
 d_{c,D}(z_e+\delta)-d_{c_*,D}(z_e+\delta)
 &\geq d_{c,D}(z_e)-L_{c,D}\|\delta\|_2\\
 &\quad-d_{c_*,D}(z_e)-L_{c_*,D}\|\delta\|_2\\
 &=\Delta_{c,D}(e)-(L_{c,D}+L_{c_*,D})\|\delta\|_2.
\end{align*}
The finite minimum defining $r_D(e)$ is positive. If
$\|\delta\|_2<r_D(e)$, the final expression is strictly positive for
\emph{every} competitor. Hence $c_*$ remains the unique minimizer and
$h_D(z_e+\delta)=c_*$. Taking $\delta=z_{e'}-z_e$ gives the statement about
membership in $Q_{c_*,D}$.
\end{proof}

\paragraph{A simpler conservative radius.}
Define the nearest-competitor margin and maximum distance sensitivity by
\[
 m_D(e)\Def\min_{c\ne c_*}\Delta_{c,D}(e),\qquad
 L_D^{\max}\Def\max_{c\in\mathcal C}L_{c,D}.
\]
Then $r_D(e)\geq m_D(e)/(2L_D^{\max})$ because every numerator is at least
$m_D(e)$ and every denominator is at most $2L_D^{\max}$.
A runner-up-specific sensitivity bound is not sufficient unless the other
competitors are also controlled. At a tie this argument provides no positive
radius. If $|\mathcal C|=1$, the hard assignment is constant everywhere.

\paragraph{What is certified.}
The result preserves the selected identity, not its semantic correctness,
and it does not assert equality of the full soft-membership vectors.
It applies to an embedding perturbation with the entire bank fixed,
including its means and precision matrices. If the embedding component
$f_\theta:\mathbb R^m\to\mathbb R^p$ of $\psi_\theta$ is
$K_\theta$-Lipschitz on a feature neighborhood, with $K_\theta>0$, the same
argument certifies feature perturbations that remain in that neighborhood
and have norm below $r_D(e)/K_\theta$.
This is conditional on that Lipschitz bound and does not establish
invariance to arbitrary raw-schema transformations.

\subsection{From Signed-Evidence Losses to a Stability Margin}
\label{app:signed-evidence-stability}

The current objective already supplies a conditional connection between
evidence quality and local stability. This connection uses its existing
losses, not a new loss or an additional training procedure. For a fixed
evidence object $e$, define its contributions to those losses by
\begin{align*}
 \ell_+(e)&\Def\sum_{c\in\mathcal C}
      \omega_D^+(e,c)d_{c,D}^2(z_e),\\
 \ell_-(e)&\Def\sum_{c\in\mathcal C}
      \omega_D^-(e,c)[\xi_- - d_{c,D}^2(z_e)]_+^2.
\end{align*}

\begin{proposition}[Conditional loss-to-margin bound]
\label{prop:signed-evidence-stability}
Suppose $|\mathcal C|\geq2$ and choose a candidate identity $c_0$ with
$w_+\Def\omega_D^+(e,c_0)>0$ and
$w_-\Def\min_{c\ne c_0}\omega_D^-(e,c)>0$.
Set
\[
 a_e\Def\sqrt{\ell_+(e)/w_+},\qquad
 b_e\Def\sqrt{\bigl[\xi_- -\sqrt{\ell_-(e)/w_-}\bigr]_+}.
\]
If $b_e>a_e$, then $c_0$ is the unique nearest identity and
\[
 r_D(e)\geq\frac{b_e-a_e}{2L_D^{\max}}>0.
\]
\end{proposition}
\begin{proof}
All loss summands are nonnegative. The summand for $c_0$ in $\ell_+$ gives
$w_+d_{c_0,D}^2(z_e)\leq\ell_+(e)$, hence $d_{c_0,D}(z_e)\leq a_e$.
For each competitor $c\ne c_0$, its opposing weight is at least $w_-$, so
\[
 w_-[\xi_- -d_{c,D}^2(z_e)]_+^2\leq\ell_-(e).
\]
Taking square roots gives
$[\xi_- -d_{c,D}^2(z_e)]_+\leq\sqrt{\ell_-(e)/w_-}$.
Since $x\leq[x]_+$, rearrangement yields
\[
 d_{c,D}^2(z_e)\geq
 \xi_- -\sqrt{\ell_-(e)/w_-}.
\]
Combining this with $d_{c,D}^2(z_e)\geq0$ gives $d_{c,D}(z_e)\geq b_e$.
Thus every competitor exceeds the candidate distance by at least
$b_e-a_e>0$. The conservative radius bound above now proves the result.
\end{proof}

Supporting loss controls proximity to the candidate, opposing loss controls
separation from covered competitors, and the precision norms control
sensitivity to displacement. Positive opposing coverage of every competitor
is essential for this particular uniform guarantee; a zero opposing weight
provides no such control. Training need not attain $b_e>a_e$, and the bound
does not establish that the evidence endorses the semantically correct identity.

\subsection{When the Prototype Realization Also Changes}
\label{app:changing-realization}

A dataset modification can change both the embedding and its prototype
bank. The fixed-bank certificate must not be used unchanged in that case.
For completeness, consider two banks in the same $\mathbb R^p$, with the
same identity labels and positive-definite precisions. Write
$A_c=\Lambda_{c,D}^{1/2}$ and $\widetilde A_c=\widetilde\Lambda_c^{1/2}$,
and let $\widetilde z$ be the modified embedding. Suppose
\[
 \|\widetilde z-z_e\|_2\leq\epsilon_z,\quad
 \|\widetilde\mu_c-\mu_{c,D}\|_2\leq\epsilon_{\mu,c},\quad
 \|\widetilde A_c-A_c\|_{\mathrm{op}}\leq\epsilon_{A,c}.
\]
Define the nonnegative distance-drift bound
\[
 \eta_c\Def (L_{c,D}+\epsilon_{A,c})
                  (\epsilon_z+\epsilon_{\mu,c})
             +\epsilon_{A,c}\|z_e-\mu_{c,D}\|_2.
\]
Indeed, the identity
\begin{align*}
 \widetilde A_c(\widetilde z-\widetilde\mu_c)
       -A_c(z_e-\mu_{c,D})
 &=\widetilde A_c\bigl[(\widetilde z-z_e)
                         -(\widetilde\mu_c-\mu_{c,D})\bigr]\\
 &\quad+(\widetilde A_c-A_c)(z_e-\mu_{c,D})
\end{align*}
and the reverse triangle inequality imply
$|\widetilde d_c(\widetilde z)-d_{c,D}(z_e)|\leq\eta_c$, where
$\widetilde d_c(z)=\|\widetilde A_c(z-\widetilde\mu_c)\|_2$.
Therefore the original unique winner $c_*$ is preserved whenever
\[
 \Delta_{c,D}(e)>\eta_c+\eta_{c_*}
 \qquad\text{for every }c\ne c_*.
\]
This extension identifies the additional quantities that must be controlled
when support is recomputed. It does not assume that arbitrary changes of
dataset automatically satisfy these bounds.

\section{Sample Execution Trace}

\newcommand{\trajcell}[4]{%
  \RaggedRight
  \textbf{#1. #2}\enspace
  \texttt{[#3]}\enspace
  #4%
}

\newcommand{\trajanswer}[1]{%
  \RaggedRight
  \textbf{Answer.}\enspace #1%
}



\refstepcounter{figure}
\label{fig:pancancer-q1-trajectories}

\begingroup
\normalsize
\setlength{\tabcolsep}{5pt}
\renewcommand{\arraystretch}{1.10}
\setlength{\LTleft}{0pt}
\setlength{\LTright}{0pt}
\setlength{\LTpre}{0pt}
\setlength{\LTpost}{0pt}

\begin{longtable}{
  @{}||
  >{\RaggedRight\arraybackslash}
    p{\dimexpr0.5\textwidth-2\tabcolsep-3\arrayrulewidth-1.5\doublerulesep\relax}
  ||
  >{\RaggedRight\arraybackslash}
    p{\dimexpr0.5\textwidth-2\tabcolsep-3\arrayrulewidth-1.5\doublerulesep\relax}
  ||@{}
}


\hline\hline
\multicolumn{2}{||p{\dimexpr\textwidth-2\tabcolsep-4\arrayrulewidth-2\doublerulesep\relax}||}{%
  \RaggedRight
  \textbf{Query.}
  For LGG patients, compute the average log10-transformed expression
  of the IGF2 gene across different histology types. Include only valid
  IGF2 values and histology annotations not enclosed in square brackets.
  Report at least four decimal places.
}
\\
\hline\hline

\textbf{Permute EQ (correct)}\par
\textbf{Summary of Prompt:}
Answer the query using only the provided evidence payload and DAB data.
Use the available tools to verify the evidence and compute the final answer.
\par
\emph{13 agent steps; 31 tool calls}
&
\textbf{Permute Core (incorrect)}\par
\textbf{Summary of Prompt:}
Answer the query without internet access using only the provided DAB files.
Use the available tools to inspect the data and compute the final answer.
\par
\emph{20 agent steps; 20 tool calls}
\\[3pt]

\hline
\endfirsthead


\hline\hline
\multicolumn{2}{||c||}{%
  \small\itshape
  PANCANCER\_ATLAS Q1 agent trajectories---continued
}
\\
\hline

\textbf{Permute EQ (correct)}
&
\textbf{Permute Core (incorrect)}
\\[2pt]

\hline
\endhead


\hline
\multicolumn{2}{r}{%
  \small\itshape Continued on next page
}
\\
\endfoot


\hline\hline
\endlastfoot


\trajcell{1}{Evidence setup}{DuckDB $\times$2}
  {Read query interpretation and diagnostics.}
&
\trajcell{1}{Exploration}{semantic probe}
  {Profile clinical categories, including histology fields.}
\\[3pt]

\trajcell{2}{Evidence setup}{DuckDB $\times$2}
  {Read selected tasks, anchors, and resolved fields.}
&
\trajcell{2}{Exploration}{semantic probe}
  {Profile RNA-seq sample types.}
\\[3pt]

\trajcell{3}{Exploration}{DuckDB $\times$2}
  {Inspect clinical and RNA-seq schemas.}
&
\trajcell{3}{Disambiguation}{semantic probe}
  {Compare the three candidate histology fields.}
\\[3pt]

\trajcell{4}{Exploration}{DuckDB $\times$2}
  {Sample patient identifiers and expression records.}
&
\trajcell{4}{Exploration}{semantic probe}
  {Search patient descriptions for cohort terms.}
\\[3pt]

\trajcell{5}{Exploration}{DuckDB $\times$2}
  {Identify the LGG cohort and verify its text pattern.}
&
\trajcell{5}{Exploration}{DuckDB}
  {Inspect the clinical schema.}
\\[3pt]

\trajcell{6}{Disambiguation}{DuckDB $\times$2}
  {Count LGG patients and inspect ICD-O-3 groups.}
&
\trajcell{6}{Exploration}{DuckDB}
  {Sample RNA-seq expression records.}
\\[3pt]

\trajcell{7}{Disambiguation}{DuckDB $\times$2}
  {Compare histology fields and bracket patterns.}
&
\trajcell{7}{Exploration}{DuckDB}
  {Sample clinical records and histology values.}
\\[3pt]

\trajcell{8}{Exploration}{DuckDB $\times$2}
  {Verify the IGF2 symbol and expression validity.}
&
\trajcell{8}{Exploration}{DuckDB}
  {Inspect diagnosis frequencies.}
\\[3pt]

\trajcell{9}{Execution}{DuckDB $\times$2}
  {Join LGG patients to IGF2 and inspect sample types.}
&
\trajcell{9}{Exploration}{DuckDB}
  {Test cohort phrase extraction.}
\\[3pt]

\trajcell{10}{Execution}{DuckDB}
  {Compute candidate aggregates at both histology grains.}
&
\trajcell{10}{Exploration}{DuckDB}
  {Count glioma and lower-grade glioma records.}
\\[3pt]

\trajcell{11}{Verification}{DuckDB $\times$2}
  {Check joined records through alternate extracts.}
&
\trajcell{11}{Exploration}{DuckDB}
  {Enumerate glioma phrases in patient descriptions.}
\\[3pt]

\trajcell{12}{Verification}{calculator $\times$5}
  {Evaluate the unshifted log transform by code.}
&
\trajcell{12}{Execution}{DuckDB}
  {Isolate LGG patients and group descriptive histology labels.}
\\[3pt]

\trajcell{13}{Execution}{calculator $\times$5}
  {Compute the selected $\log_{10}(x+1)$ means by code.}
&
\trajcell{13}{Exploration}{DuckDB}
  {Inspect IGF2 coverage and sample types.}
\\[3pt]

\trajanswer{%
  \texttt{9382/3} 2.7136;
  \texttt{9400/3} 2.6014;
  \texttt{9401/3} 2.5584;
  \texttt{9450/3} 2.6967;
  \texttt{9451/3} 2.5826.
}
&
\trajcell{14}{Execution}{DuckDB}
  {Join LGG patients to valid IGF2 records.}
\\[4pt]

&
\trajcell{15}{Disambiguation}{DuckDB}
  {Compute candidates for all three histology fields.}
\\[3pt]

&
\trajcell{16}{Verification}{DuckDB}
  {Recompute descriptive-label means with an alternate join.}
\\[3pt]

&
\trajcell{17}{Verification}{Python}
  {Independently reproduce the descriptive-label result.}
\\[3pt]

&
\trajcell{18}{Disambiguation}{semantic probe}
  {Reinspect the candidate histology fields.}
\\[3pt]

&
\trajcell{19}{Execution}{DuckDB}
  {Finalize the three-label aggregation.}
\\[3pt]

&
\trajcell{20}{Selection}{result gate}
  {Select the internally consistent result at the wrong grain.}
\\[4pt]

&
\trajanswer{%
  Oligoastrocytoma 2.7136;
  Oligodendroglioma 2.6825;
  Astrocytoma 2.5713.
  The required five ICD-O-3 groups are not reported.
}
\\[3pt]

\end{longtable}

\vspace{3pt}

\noindent
\textbf{Figure~\thefigure.}
Observed agent trajectories for PANCANCER$\_$ATLAS Q1. Each numbered item represents one agent step; multipliers denote tool calls issued in parallel within that step. Both systems correctly identify the LGG cohort, join the clinical and expression data, and compute the requested transformation. They diverge at the grouping-field decision: Permute EQ selects the required ICD-O-3 histology groups, whereas Permute Core aggregates the same underlying records into three broader descriptive histology labels. Its numerical computation is therefore consistent, but its output is incorrect because it is reported at the wrong histology grain and omits the five required ICD-O-3 groups. The summaries describe visible actions and results, not hidden chain-of-thought. This selected example does not represent aggregate performance.

\endgroup

\begin{figure}[H]
  \centering
  \includegraphics[width=0.8\linewidth]{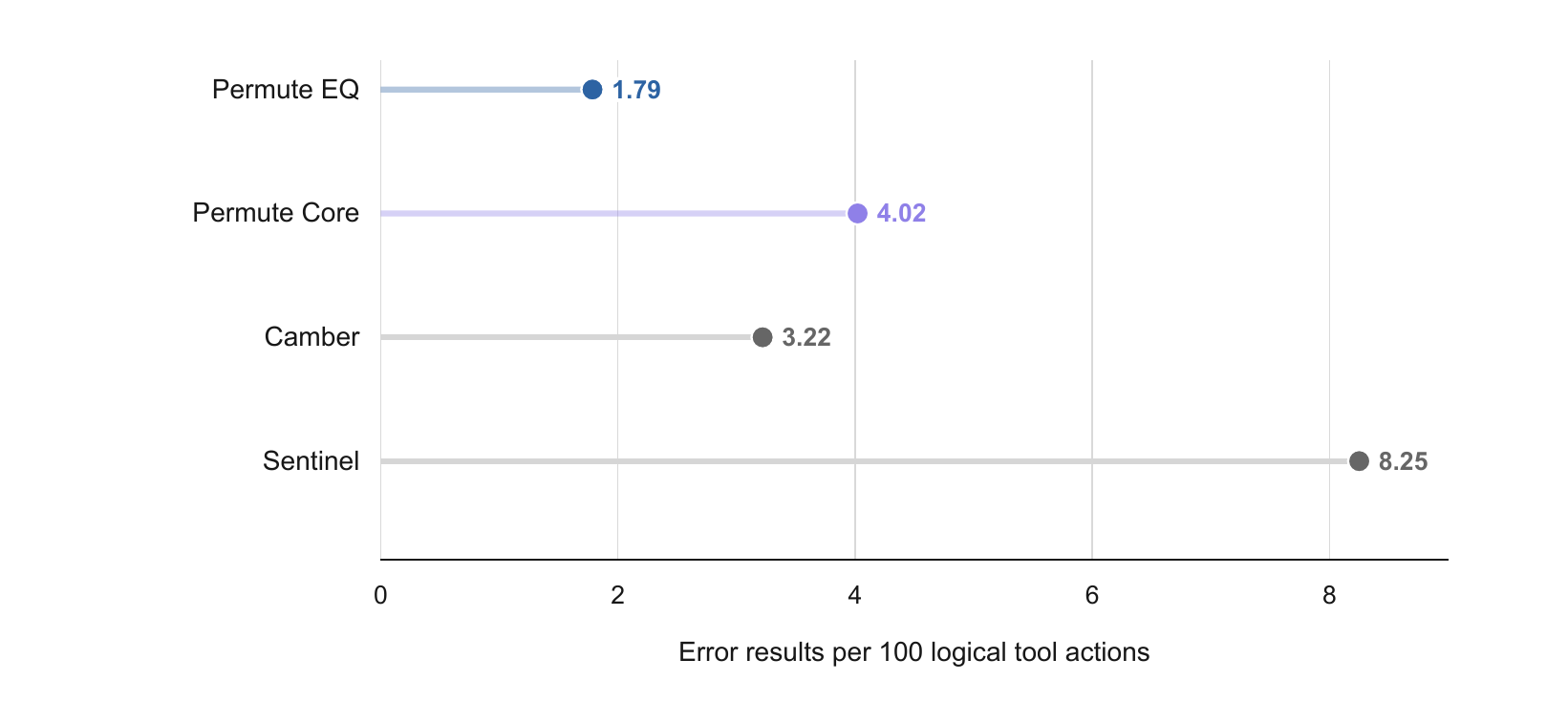}
  \caption{
    Aggregated tool errors taken from the tool traces
  }
  \label{fig:tool-errors-observed}
\end{figure}

\end{document}